\documentclass{article}
\usepackage{fullpage}
\usepackage[utf8]{inputenc} 
\usepackage[T1]{fontenc}    
\usepackage{hyperref}
\usepackage{nicefrac}       
\usepackage{amssymb}
\usepackage{color}
\usepackage{amsthm}
\usepackage{mathtools}
\usepackage{subcaption}
\usepackage{caption}
\usepackage{tcolorbox}
\usepackage{titlesec}
\usepackage{stackengine}
\newtheorem{theorem}{Theorem}

\newtheorem{lemma}[theorem]{Lemma}

\usepackage{tikz}

\DeclareMathAlphabet{\pazocal}{OMS}{zplm}{m}{n}

\usepackage{thmtools, thm-restate}

\title {Fair Division of Multi-layered Cakes}

\author{ {Mohammad Azharuddin Sanpui} \\
Department of Mathematics\\
Indian Institute of Technology Kharagpur\\
Kharagpur-721302,India
}
\date{}

\begin{document}
\maketitle

\begin{abstract}
We consider multi-layered cake cutting in order to fairly allocate numerous divisible resources (layers of cake) among a group of agents under two constraints: contiguity and feasibility. We first introduce a new computational model in a multi-layered cake named ``a pair of knives''. Then, we show the existence of an exact multi-allocation for two agents and two layers using the new computational model. We demonstrate the computation procedure of a feasible and contiguous proportional multi-allocation over a three-layered cake for more than three agents. Finally, we develop a technique for computing proportional allocations for any number $n\geq 2^a3$ of agents and $2^a3$ layers, where $a$ is any positive integer.
\end{abstract}

\section{Introduction}\label{sec:introduction}
There are several instances of time scheduling in our daily lives where we arrange our schedules so that we can finish our daily tasks. Consider a group of university students who desire to enjoy several facilities, such as a seminar lecture or an indoor game. The two facilities have the same opening and closing hours. Everyone in the group of students is willing to enjoy both facilities, but each has a distinct preferred time period for taking each one.\newline
In simple terms, the problem of dividing a cake is a metaphor for how to divide a resource that can be shared among $n$ agents with different preferences in a fair way. The cake-cutting problem is a central topic in the theory of fair division \cite{1brams1996fair,4moulin2004fair,2robertson1998cake,3brandt2016handbook}, and it has received a significant amount of attention in the domains of mathematics, economics, political science, and computer science \cite{5edmonds2011cake,6aumann2015efficiency,7caragiannis2012efficiency,8thomson2007children,9procaccia2013cake,10branzei2022query}. It is hard to give each agent a fair share of the cake.
\textit{Envy-freeness} and \textit{proportionality} are the most important criteria for a fair allocation in the cake-cutting literature. In an envy-free allocation, every agent is pleased with the pieces they are allocated as opposed to any other agent's allocation. In a proportional allocation, each agent receives at least $\frac{1}{n}$ of the value he estimates for the cake. When all of the cake has been divided, envy-freeness entails proportionality. It is generally known that envy-free allocations always exist \cite{11brams1995envy}, even if we specify that each agent must receive a connected piece \cite{12edward1999rental,13stromquist1980cut}. In addition to its existence, the algorithmic design aspect of the process has also been thought about for a long time \cite{14dubins1961cut,15aziz2016discrete,16aziz2016discrete,17,18even1984note,19stromquist2008envy}. For any number of agents, we are able to calculate a proportional allocation as well as an envy-free allocation.
\newline
We cannot consider the problem of getting several facilities in the example above to be a cake-cutting problem. We have to divide the two time intervals independently so that each agent can enjoy both facilities. Now the issue is how to fairly divide each facility's time interval in accordance with their preferences. As a result, every student can enjoy every facility, and the allotted time intervals for each facility never overlap. Adopting the above constraints, Hosseini et al. \cite{20Hosseini2020FairDO} initiate the multi-layered cake cutting problem. In the multi-layered cake cutting problem, we see how to solve this problem. We consider each facility as a divisible heterogeneous layer of a multi-layered cake. Every student has an additive preference, called \emph{valuation} over disjoint (non-overlapping) intervals.
	Note that the valuation of the same 
	part of the cake for different 
	students can be very different.
	A division of a multi-layered cake is feasible if no student’s time intervals contain overlapping intervals. The division is contiguous if each student gets a contiguous time interval for each facility. Our goal in multi-layered cake cutting is to find multi-allocations that are fair while also meeting the constraints of feasibility and contiguity.
\subsection{Our results}\label{subsec:our results}
In section $3$, we show the existence of exact feasible multi-allocation for two agents using the idea of the Austin-moving knife procedure. In section $4$, we show that there exists a proportional multi-allocation that is feasible and contiguous for three layers and any number $n\geq3$ of agents. We also prove the existence of a proportional multi-allocation that is feasible and contiguous for $2^a3$ layers and for $n\geq 2^a3$ agents, where $a\in \mathbb{Z}_+$. 
\begin{table}[htbp]
 \caption{Computational procedures of finding fair multi-allocation for different layers and agents are shown in the following table.}		
 		\centering
 		\begin{tabular}{|c |c
 				|c |c |c|c| }
 			
 			\hline 
 			Agents($n$) & Layers(m)  & EF & Prop    \\ [0.5ex] 
 			\hline
 			
 			2 & 2   &\cite{20Hosseini2020FairDO}  &  \cite{20Hosseini2020FairDO}\\ 
 			\hline
 			3 & 2  & \cite{20Hosseini2020FairDO}  & \cite{20Hosseini2020FairDO}   \\
 			\hline
 			 $n\geq m$ & $2^a$   & - & \cite{20Hosseini2020FairDO} \\
 			\hline
 		3 & 3 &- &  Theorem \ref{thm: pro3}\\
 			\hline
     $n\geq3$ & 3  & -&  Theorem \ref{thm: prograt3}\\
 			\hline
    $2^a3$ & $2^a3$  &- &  Theorem \ref{thm: 23}\\
 			\hline
    $n\geq 2^a3$ & $2^a3$  &- &  Theorem \ref{thm: grat23}\\
 			\hline
 		\end{tabular}
 
 	\end{table}

\subsection{Related Work}\label{subsec:related work}
Cake cutting problem is the central topic in the fair division. In recent years, it has been extensively studied in economics, mathematics, and computer science literature \cite{21caragiannis2019unreasonable,22procaccia2009thou,23deng2012algorithmic,24barbanel2005geometry,Aumann2012ComputingSC}. In order to fairly divide multiple divisible resources among a group of agents, Hosseini et al. \cite{20Hosseini2020FairDO} initiate the study of multi-layered cake cutting. They show that there exists an envy-free multi-allocation that is feasible and contiguous for two layers and three agents with two types of preferences. Igarashi and Meunier \cite{25igarashi2021envy} also show that an envy-free multi-allocation exists that is feasible and contiguous when the number of agents is a prime power and the number of layers is at most the number of agents, using combinatorial topology. There are a few papers most related to multi-layered cake cutting, where agents can simultaneously benefit from all allocated pieces with no constraints \cite{26lebert2013envy,27nyman2020fair,28cloutier2010two}.
\section{Preliminaries}\label{sec:preli}
We take into account the multi-layered cake cutting notion that Hosseini et al. \cite{20Hosseini2020FairDO} developed. The number of layers and agents are specified in the problem of cutting a multi-layered cake. The setting of the model includes a set $\mathcal{N}=\{1,2,....,n\}$ of agents and a set $\mathcal{M}=\{1,2,...,m\}$ of layers. An $m$-layered cake is denoted by \(\mathfrak{C}=(L_l)_{l\in \mathcal{M}}\) where $L_l$ is an interval such that $L_l\subseteq [0,1]$ for each \(l\in \mathcal{M}\). Due to each \(l\in \mathcal{M}\), we allude $l$ as $l$-th layer and $L_l$ as $l$-th layer cake.\newline
Corresponding to each  $l$-th layer, each agent \(i \in \mathcal{N}\) is endowed with a non-negative integrable density function $v_{il}:L_l\rightarrow \mathbb{R}_0^+$ where \(l\in \mathcal{M}\). The valuation function of each agent $i$ for $l$-th layer $L_l$ is a function representing the preference of the agent $i$ over different parts of $L_l$ denoted by $V_{il}:2^{L_{l}} \rightarrow \mathbb{R}^+_0$.

If $P$ is a piece of cake of the $l$-th layer, then $V_{il}(P)$ represents the value assigned to it by the agent $i$ i.e, $V_{il}(P)=\sum_{I\in P}\int_{x\in I} v_{il}(x) dx$. We consider the total valuation of each agent $i$ over the entire cake is 1 i.e, $\sum_{1\leq l\leq m} V_{il}(L_l)=1$. \newline
A sequence \(P=(P_l)_{l\in \mathcal{M}}\) of pieces of each layer \(l\in \mathcal{M}\) of multi-layered cake is called a layered piece. A layered piece \(P=(P_l)_{l\in \mathcal{M}}\) is said to be \textit{contiguous} if each \(l\in \mathcal{M}\),  $P_l$ is a \textit{contiguous} piece of the layer $l$. A layered piece \(P=(P_l)_{l\in \mathcal{M}}\) is said to be \textit{non-overlapping}, if for any two pieces from different layers never overlap i.e., for any two different layers \(l,l^\prime \in \mathcal{M}\) and for any $I\in P_l$ and $I^\prime \in P_{l^\prime}$ either $I\cap I^\prime=\emptyset$ or  $I\cap I^\prime=\{b\}$ where $b$ is one of the endpoints of both $I$ and $I^\prime$. We assume the valuation functions related to layers of each agent $i$  are additive over layers and written as $V_i(P)=\sum_{1\leq l \leq m} V_{il}(P_l)$, where $V_i$ is the valuation function of each agent $i$ over the entire cake where \(i\in \mathcal{N}\).
\newline
Let $P$ and \(P^\prime\) be two layered pieces. If for any \(i\in \mathcal{N}\), $V_i(P)\geq {V_i}(P^\prime)$ then it is said that the agent $i$ weakly prefers the layered piece $P$ to the layered piece $P^\prime$. A multi-allocation \(\mathcal{P}=(\mathcal{P}_1,\mathcal{P}_2,....,\mathcal{P}_n)\) is a partition of the multi-layered cake where each \(\mathcal{P}_i=(P_{il})_{l\in \mathcal{M}}\) is a layered piece of the cake that is assigned to agent $i$ where \(i\in \mathcal{N}\).
Corresponding to a multi-allocation \(\mathcal{P}\) and an agent \(i\in \mathcal{N}\), the valuation of each agent $i$ is \(V_i(\mathcal{P}_i)=\sum_{l\in \mathcal{M}} V_{il}(P_{il})\). A multi-allocation \(\mathcal{P}\) is said to be complete if for each \(l\in \mathcal{M}\), $\bigcup_{i=1}^n P_{il}=L_l$. A multi-allocation \(\mathcal{P}=(\mathcal{P}_1,\mathcal{P}_2,....,\mathcal{P}_n)\) is said to be 
\begin{itemize}
    \item contiguous if for each \(i\in \mathcal{N}\), \({\mathcal{P}}_i\) is contiguous.
    \item feasible if for each \(i\in \mathcal{N}\), \({\mathcal{P}}_i\) is non-overlapping. 
\end{itemize}
\subsection{Fairness notions}\label{subsec:fairness}
\textbf{Definition 1.} A multi-allocation $\mathcal{P}=(\mathcal{P}_1,\mathcal{P}_2,.....,\mathcal{P}_n)$ is said to be exact if for any two agents $i,j\in \mathcal{N}$,
$V_i(\mathcal{P}_j)=\frac{1}{n}$.
\newline
\textbf{Definition 2.} A multi-allocation $\mathcal{P}=(\mathcal{P}_1,\mathcal{P}_2,.....,\mathcal{P}_n)$ is said to be proportional if for any agent $i\in \mathcal{N}$,
$V_i(\mathcal{P}_i)\geq \frac{1}{n}$.
\newline
\textbf{Definition 3.} A multi-allocation $\mathcal{P}=(\mathcal{P}_1,\mathcal{P}_2,.....,\mathcal{P}_n)$ is said to be envy-free if for any two agents $i,j\in \mathcal{N}$,
$V_i(\mathcal{P}_i)\geq V_i(\mathcal{P}_j)$.
\subsection{The m (even)-layered cut}\label{subsec: cut} 
Hosseini et al. \cite{20Hosseini2020FairDO} define a partition of a multi-layered cake $\mathfrak{C}$ into two diagonal pieces that follows the feasibility constraint. For any $x\in I=[0,1]$, defined :
     \begin{enumerate}
         \item  \(LR(x,\mathfrak{C})=(\bigcup_{l=1}^{\frac{m}{2}}{L}_l \cap [0,x])\cup 
         (\bigcup_{l=\frac{m}{2} +1}^{m}{L}_l \cap [x,1]) \)
         
         \item \(RL(x,\mathfrak{C})=(\bigcup_{l=1}^{\frac{m}{2}}{L}_l \cap [x,1])\cup 
         (\bigcup_{l=\frac{m}{2} +1}^{m}{L}_j \cap [0,x])\)  
     \end{enumerate}
    
    \begin{itemize}
        \item $LR(x,\mathfrak{C})$ consists of all subintervals  of type $[0,x]$ of each layer \(l\in\{1,...,\frac{m}{2}\}\)  and  subintervals of type $[x,1]$ of each layer \(l \in\{\frac{m}{2}+1,...,m\}\) .
    
        \item $RL(x,\mathfrak{C})$ consists of all subintervals of type $[0,x]$ of each layer \(l \in\{\frac{m}{2}+1,...,m\}\)  and  subintervals of type $[x,1]$ of each layer \(l\in\{1,...,\frac{m}{2}\}\) .
    \end{itemize}
    \begin{center}
    \begin{tikzpicture}
    \draw[thick,-](0,0)--(5,0);
    \draw[-](0,1.5)--(5,1.5);
    \draw[-](0,.75)--(5,.75);
    \draw[-](0,0)--(0,1.5);
    \draw[-](5,0)--(5,1.5);
    \draw[dashed](0,.375)--(5,.375);
    \draw[dashed](0,1.125)--(5,1.125);
    \draw[-](2.5,0)--(2.5,1.5);
    \node at (0,1.8) {$0$};
    \node at (2.5,1.8){$x=\frac{1}{2}$};
    \node at (5,1.8) {$1$};
    \node at (3.75,.375) {$LR(x,\mathfrak{C})$};
    \node at (1.25,1.125) {$LR(x,\mathfrak{C})$};
    \node at (3.75,1.125) {$RL(x,\mathfrak{C})$};
    \node at (1.25,.375) {$RL(x,\mathfrak{C})$};
    \node at (-.5,.1875) {$L_4$};
    \node at (-.5,.5625) {$L_3$};
    \node at (-.5,.9375) {$L_2$};
    \node at (-.5,1.3125){$L_1$};
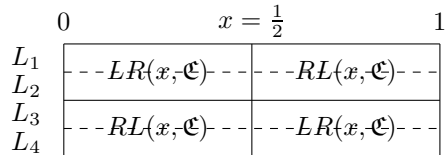
    \node at (2.5,-.3){Figure 1: Diagonal pieces when $x=\frac{1}{2}$ and $m=4$.};
\end{tikzpicture}
  \end{center}
    
The merge of $LR(x)$ is an $\frac{m}{2}$-layered cake  whose $l$-th layer piece is defined $S_l\cup S_{l+\frac{m}{2}}$ where $S_l=[0,x]\cap L_l$ and $S_{l+\frac{m}{2}}=[x,1]\cap L_{l+\frac{m}{2}}$ for $l\in \{1,2,....,\frac{m}{2}\}$. Similar for $RL(x)$. We use $LR(x)$ and $RL(x)$ in the place of $LR(x,\mathfrak{C})$ and $RL(x,\mathfrak{C})$, respectively.
\subsection{Computational model}\label{subsec:computational model}
In the cake cutting problem, Robertson-Webb query model plays an important role \cite{2robertson1998cake}. Following Robertson–Webb query model, Hosseini et al. \cite{20Hosseini2020FairDO} introduce a new computational model in the multi-layered cake cutting problem where there are two types of queries, one is a short knife and the other is long knife . \newline
1. \textbf{Short Knife:}
\textit{Short evaluation query:} For any given interval $[x,y]$ of $l$-th layer of an $m$-layered cake $\mathfrak{C}$, ${eval}_l(i,x,y)$ implies the valuation $V_{il}([x,y])$ of the agent $i$ in the interval $[x,y]$ of $l$-th layer cake ${L}_l$ . Here $[x,y]\subseteq {L}_l \subseteq  [0,1]$.
\textit{Short cut query:} For any given point $x$ on $l$-th layered cake ${L}_l$ and $p\in [0,1], $ ${cut}_l(i,x,p)$ indicates the minimum point $y$ on the $l$-th layered cake ${L}_l$ for which $V_{il}([x,y])=p$.\\
2. \textbf{Long Knife:}
\textit{Long evaluation query:} For any given point $x\in [0,1],$ $eval(i,x)$ implies the valuation $V_i(LR(x))$ of the agent $i$ for $LR(x)$.
\textit{Long cut query:} For any given $p\in [0,1],$ $cut(i,p)$ indicates the minimum point $x$ for which the valuation of the agent $i$ should be $p$ for a piece of cake $LR(x)$ if such a point $x$ exists.
\subsection{Switching point}\label{subsec: switching}
We find a point $x$ in the interval $[0,1]$ that divides the entire $m$-layered cake $\mathfrak{C}$ into a pair of diagonal pieces, $LR(x)$ and $RL(x)$, such that both have the same valuation for a particular agent, i.e., if there exists a point $x\in [0,1],$ for which the valuation of an agent $i$ for two pairs of diagonal pieces is the same, i.e., $V_i(LR(x))=V_i(RL(x))$. Then the point $x$ is called a switching point over the entire cake $\mathfrak{C}$ for the agent $i$. Hosseini et al. [20] first acknowledged the existence of such a type of point.
\vspace{0.03cm}\newline
A set of agents \(\mathcal{S}\) is said to be a majority set if \(|\mathcal{S}|\geq \lceil{\frac{n}{2}}\rceil\). Suppose \(\mathcal{A}_p\) and \(\mathcal{A}_q\) are two layered pieces, and \(\mathcal{S}\) is a majority set. If every agent \(i\in \mathcal{S}\), weakly prefers \(\mathcal{A}_p\) to \(\mathcal{A}_q\), we say that a majority weakly prefers \(\mathcal{A}_p\) to \(\mathcal{A}_q\) and denote it by \(\mathcal{A}_p\mathrel{\stackon[1pt]{$\succeq$}{$\scriptstyle m$}} \mathcal{A}_q\). A point $x\in [0,1]$ is said to be a majority switching point over an $m$-layered cake $\mathfrak{C}$, if 
$LR(x) \mathrel{\stackon[1pt]{$\succeq$}{$\scriptstyle m$}} RL(x)$ and 
$RL(x) \mathrel{\stackon[1pt]{$\succeq$}{$\scriptstyle m$}}LR(x)$. 
Hosseini et al. \cite{20Hosseini2020FairDO} show that there exists a majority switching point over an $m$-layered cake for any number $n\geq m$ of agents, where $m$ is even.
\section{Exact multi-layered cake cutting}\label{sec:exact}
We analyze the challenge of attaining a complete exact multi-allocation for a pair of agents on a two-layered cake. The Austin moving-knife procedure for two agents provides the fundamental concept needed for proving the existence of exact multi-allocation. The intermediate value theorem is the primary mathematical tool used by the Austin moving-knife procedure \cite{29austin1982sharing}. Simply speaking, we reach our goal by continuously moving a pair of knives (defined below), taking advantage of the intermediate value theorem.
\vspace{0.1cm}\newline
We provide a particular approach for partitioning the layered cake in order to meet the non-overlapping criterion when cutting it. While showing the existence of exact multi-allocation, we take advantage of this partition. 
\newline
\textbf{ The m (even)-layered cut.}\\
We define a partition for an $m$-layered cake that satisfies the non-overlapping constraint, where $m$ is an even number. For any two points $x$ and $y$ in the interval $[0,1]$, where $ x\leq y$, we define
\begin{itemize}
    \item $TLR(x,y)=(\bigcup_{j=1}^{\frac{m}{2}} L_j\cap [0,x])\cup (\bigcup_{j=1}^{\frac{m}{2}} L_{\frac{m}{2}+j}\cap [x,y])\cup (\bigcup_{j=1}^{\frac{m}{2}} L_j\cap [y,1] )$;
    \item $TRL(x,y)=(\bigcup_{j=1}^{\frac{m}{2}} L_{\frac{m}{2}+j}\cap [0,x])\cup (\bigcup_{j=1}^{\frac{m}{2}} L_j\cap [x,y])\cup (\bigcup_{j=1}^{\frac{m}{2}} L_{\frac{m}{2}+j}\cap [y,1])$.
\end{itemize}

\begin{center}
    \begin{tikzpicture}
    \draw[thick,-](0,0)rectangle (6,3);
    \draw[dashed](0,.5)--(6,.5);
    \draw[dashed](0,1)--(6,1);
    \draw[dashed](0,2)--(6,2);
    \draw[dashed](0,2.5)--(6,2.5);
    \draw[-](0,1.5)--(6,1.5);
    \draw[-](2,0)--(2,3);
    \draw[-](4,0)--(4,3);
    \node at (0,3.3){$0$};
    \node at (2,3.3){$x=\frac{1}{3}$};
    \node at (4,3.3){$y=\frac{2}{3}$};
    \node at (6,3.3){$1$};
    \node at (1,.75){$TRL(x,y)$};
    \node at (1,2.25){$TLR(x,y)$};
    \node at (3,.75){$TLR(x,y)$};
    \node at (3,2.25){$TRL(x,y)$};
    \node at (5,.75){$TRL(x,y)$};
    \node at (5,2.25){$TLR(x,y)$};
    \node at (-.3,.25) {$L_6$};
    \node at (-.3,.75) {$L_5$};
    \node at (-.3,1.25) {$L_4$};
    \node at (-.3,1.75) {$L_3$};
    \node at (-.3,2.25) {$L_2$};
    \node at (-.3,2.75) {$L_1$};
    \end{tikzpicture}
    
\end{center}
\begin{center}
Figure 2: Example of the partition induced by $x=\frac{1}{3}$ and $y=\frac{2}{3}$ for a six-layered cake.  
\end{center}
Look at the $m$-layered cut, where $m$ is equal to 2. With a 2-layered cut, a 2-layered cake gets divided into two portions that are 
$TLR(x,y)=([0,x]\cap L_1)\cup ([x,y]\cap L_2)\cup ([y,1] \cap L_1 )$ and $TRL(x,y)=([0,x]\cap L_2)\cup ([x,y]\cap L_1)\cup ([y,1]\cap L_2)$, where $x, y \in [0,1]$ and $x\leq y$.\newline
When $x=0$ and $0<y<1$, then $ TLR(0,y)=RL(y)$ and $TRL(0,y)=LR(y)$. When $0<x<1$ and $y=1$, then $TLR(x,1)=LR(x)$ and $TRL(x,1)=RL(x)$.

\begin{center}

\begin{tikzpicture}
    \draw[thick,-] (0,0)--(3,0);
    \draw[-](0,1)--(3,1);
    \draw[-](0,.5)--(3,.5);
    \draw[-](0,0)--(0,1);
    \draw[-](3,0)--(3,1);
    \draw[-](1.5,0)--(1.5,1);
    \node at (0,1.2) {$x=0$};
    \node at (1.5,1.2){$y=\frac{1}{2}$};
    \node at (-.4,.25){$L_2$};
    \node at (-.4,.75){$L_1$};
    \node at (.75,.75){$LR(y)$};
    \node at (.75,.25){$RL(y)$};
    \node at (2.25,.25){$LR(y)$};
    \node at (2.25,.75){$RL(y)$};
    
    \draw[-] (4,0)--(7,0);
    \draw[-](4,1)--(7,1);
    \draw[-](4,.5)--(7,.5);
    \draw[-](4,0)--(4,1);
    \draw[-](7,0)--(7,1);
    \draw[-](5.5,0)--(5.5,1);
    \node at (5.5,1.2) {$x=\frac{1}{2}$};
    \node at (7,1.2) {$y=1$};
    \node at (3.5,.25){$L_2$};
    \node at (3.5,.75){$L_1$};
    \node at (4.75,.75){$LR(x)$};
    \node at (4.75,.25){$RL(x)$};
    \node at (6.25,.25){$LR(x)$};
    \node at (6.25,.75){$RL(x)$};
\end{tikzpicture}

{
Figure 3: Examples of the partitions for two pairs $(x=0,y=\frac{1}{2})$ and $(x=\frac{1}{2},y=1)$.
}
\end{center}

We propose a query model that is comparable to the long knife that Hosseini et al. [20] described for cutting multi-layered cakes.\newline
\textbf{Computational model.} We propose the pair of knives query, which takes its cues from the Austin moving-knife procedure, to demonstrate the existence of exact multi-allocation.
\newline
\textit{\textbf{Pair of knives} :} \textit{Pair evaluation query}: For any pair of points $x,y\in [0,1]$ where $x\leq y$, $eval(i,x,y)$ implies the valuation $V_i(TLR(x,y))$ of the agent $i$ for $TLR(x,y)$. \textit{Pair cut query}: For any known $p\in [0,1]$, $cut(i,p)$ indicates a pair of points $(x,y)$ for which the valuation of the agent $i$ should be $p$ for the piece of cake $TLR(x,y)$, where $x, y \in [0,1]$ and $x\leq y$.
\vspace{.1cm}\newline
Suppose there are two knives with the designations $K_1$ and $K_2$, respectively. If $x$ denotes where knife $K_1$ is located and $y$ indicates where knife $K_2$ is located, then the two knives $K_1$  and $K_2$ are referred to as a pair of knives.
\newline
\textbf{Similarities between the \textit{pair of knives} and the \textit{long knife}} :\newline
 A pair of knives is comparable to a long knife when one of the knives is positioned at one of the unit interval's endpoints and the other is positioned in its interiors.
 \vspace{0.02cm}\newline
Now we show the existence of exact multi-allocation over any two-layered cake for two agents.
\begin{theorem}\label{thm:exact}
There exists an exact complete multi-allocation over any two-layered cake for two agents that satisfies the feasibility condition.
\end{theorem}
\begin{proof}
The Austin moving-knife procedure bears the fingerprints of the proof. 
In the Austin moving-knife procedure, we initially start by moving two knives over a single-layered cake from positions $0$ and $p$ so that the value of the piece between the two knives is always half with respect to agent 1, where the point $p$ divides the cake into equal-valued pieces with respect to agent 1. The movement of those two knives will come to an end at points $p$ and $1$, respectively, since point $p$ divides the cake into two equal-valued portions with regard to agent 1. When the second agent believes the value of the piece between the two knives is half, he will order the movement of the knives to stop. The intermediate value theorem demands that this situation occur.
\vspace{0.03cm}\newline
We begin to move the knives $K_1$ and $K_2$ from locations $0$ and $s_1$ in a manner similar to the Austin moving-knife procedure such that the value of the piece $TLR(0, s_1)$ is half for agent 1, where $s_1$ is a switching point for agent 1. We continuously move these two knives such that the piece $TLR(x,y)$ is valued at half with respect to agent 1, where $x\leq y$. Given that the piece $TLR(x,y)$ is still valued at half with regard to agent 1, the terminal locations of the knives $K_1$ and $K_2$ are at points $s_1$ and $1$, respectively. 
While we continuously move these two knives, keeping the value on the piece $TLR(x,y)$ always half with regard to agent 1, agent 2 can sometimes say "stop" when he thinks the value on the piece $TLR(x,y)$ is half. The intermediate value theorem implies that this situation occurs.

\begin{center}
    
{
 \begin{tikzpicture}
         \draw[thick,-] (0,0)--(5,0);
         \draw[-](0,1)--(5,1);
         \draw[-](0,0)--(0,1);
         \draw[-](5,0)--(5,1);
         \draw[-](0,.5)--(5,.5);
         \draw[-](1.8,0)--(1.8,1);
         \draw[-](3.5,0)--(3.5,1);
         \node at (.9,.75) {$TLR$};
         \node at (.9,.25) {$TRL$};
         \node at (2.65,.25) {$TLR$};
         \node at (2.65,.75){$TRL$};
         \node at (4.25,.75){$TLR$};
         \node at (4.25,.25){$TRL$};
         \node at (-.4,.25){$L_2$};
         \node at (-.4,.75){$L_1$};
         \node at (1.8,1.2){$x$};
         \node at (3.5,1.2){$y$};
\end{tikzpicture} 
}

Figure 4: When the exactness happens for a pair of points (x,y) where $TLR(x,y)=TLR$ and $TRL(x,y)=TRL$.
\end{center}
\end{proof}
\section{Proportional multi-layered cake cutting}\label{sec:proportional}
Igarashi and Meunier \cite{25igarashi2021envy} use combinatorial topology to show that a proportional multi-allocation exists and is feasible and contiguous for any number $m$ of layers and any number $n\geq m$ of agents. Our main goal in this study is to demonstrate the existence of fair multi-allocation using the computational model that Hosseini et al. \cite{20Hosseini2020FairDO} proposed. Using the cut-and-eval queries proposed by Hosseini et al., we demonstrate the computation of a proportional multi-allocation that is feasible and contiguous for three layers and any number $n\geq 3$ of agents. In their work, Hosseini et al. \cite{20Hosseini2020FairDO} leave this as an open question.
\vspace{0.04cm}\newline
We use the following lemma to prove Theorem \ref{thm: pro3}.
\begin{lemma}\label{lem: switch}
 Suppose that the number $m$ of layers is even. Take
any $i\in \mathcal{N}$. Let $r\in \mathbb{R}$ be such that $(i)$ $V_i(LR(0))\geq r$ and $V_i(RL(0))\leq r$, or $(ii)$ $V_i(LR(0)) \leq r$ and $V_i(RL(0)) \geq r$. Then, there exists a point $x\in [0,1]$ such that $i$ values $LR(x)$ exactly at $r$, i.e. $V_i(LR(x)) = r$. In particular, a switching point for $i$ always exists \cite{20Hosseini2020FairDO}.   
\end{lemma}
\begin{theorem}\label{thm: pro3}
A proportional complete multi-allocation that is feasible and contiguous exists for three layers and three agents.  
\end{theorem}
\begin{proof}
Assume that $\mathfrak{C}=(L_l)_{l\in \mathcal{M}}$ is a three-layered cake, and that $V_1$,$V_2$ and $V_3$ are the valuation functions for agents 1, 2, and 3 in that order. Due to $V_1(L_1)+V_1(L_2)+V_1(L_3)=1$, without loss of generality, we assume that agent 1 has a valuation of at least $\nicefrac{1}{3}$ over one of the first two layers and at most $\nicefrac{1}{3}$ over the other of the first two layers. So there exists a point y such that $V_1(LR(L_1\cup L_2,y))=\nicefrac{1}{3}$ by Lemma \ref{lem: switch}.

A new cake, $\mathfrak{C}^\prime$, is now defined as $(L_1^\prime,L_2^\prime)$, the status of $L_1^\prime=(L_2\cap [0,y])\bigcup (L_1\cap [y,1])$ and $L_2^\prime=L_3$. As of right now, agent $1$ has valuation $\nicefrac{2}{3}$ over new cake $\mathfrak{C}^\prime$. Due to Lemma $\ref{lem: switch}$, we find a point $z$ such that $V_1(LR(z,\mathfrak{C}^\prime))=V_1(RL(z,\mathfrak{C}^\prime))=\nicefrac{1}{3}$ where either $y\leq z$ or $z \leq y$. We design a multi-allocation \( \mathcal{P} \)= (\( \mathcal{P}_1,\mathcal{P}_2,\mathcal{P}_3 \)), where the layered pieces $\mathcal{P}_1$, $\mathcal{P}_2$, and $\mathcal{P}_3$ are made from, respectively, $LR(L_1\cup L_2,y)$, $LR(\mathfrak{C}^\prime,z)$, and $RL(\mathfrak{C}^\prime,z)$. As a result, \(V_1(\mathcal{P}_1)=V_1(\mathcal{P}_2)=V_1(\mathcal{P}_3)=\nicefrac{1}{3}\).
The constructed multi-allocation $\mathcal{P}$ must differ depending on where $y$ and $z$ are. We obtain the required multi-allocation \(\mathcal{P}^\ast \) by using this multi-allocation \(\mathcal{P}\).\newline
\textbf{Case 1:} When $z\leq y$, \(\mathcal{P}_1=(L_1\cap [0,y], L_2 \cap [y,1],\emptyset )\), \(\mathcal{P}_2=(\emptyset, L_2\cap [0,z],L_3\cap [z,1])\) and \(\mathcal{P}_3=(L_1\cap [y,1],L_2\cap [z,y],L_3\cap[0,z])\) are the corresponding layered pieces. If two distinct layered pieces \(\mathcal{P}_p \) and \(\mathcal{P}_q \) are obtained, such that
\(V_2(\mathcal{P}_p)\geq \nicefrac{1}{3}\) and \(V_3(\mathcal{P}_q )\geq \nicefrac{1}{3}\), then allocate 
\(\mathcal{P}_p \) to agent $2$, \(\mathcal{P}_q \) to agent $3$, and the remaining layered piece to agent $1$, where $p\neq q \in \{1,2,3\}$. Thus, we obtain a proportional multi-allocation $\mathcal{P}^\ast=\mathcal{P}$ that is feasible and contiguous.\newline
Otherwise, there is only one unique layered piece \( \mathcal{P}_r\) with a value larger than 1/3 for agents 2 and 3, where $r\in \{1,2,3\}$, if there are no two distinct layered pieces \(\mathcal{P}_p \) and \(\mathcal{P}_q \) are obtained in which \(V_2(\mathcal{P}_p)\geq \nicefrac{1}{3}\) and \(V_3(\mathcal{P}_q )\geq \nicefrac{1}{3}\).

 \begin{center} 
\begin{tikzpicture}
  \draw [thick,-] (-2,0)--(4,0);
 \draw[-] (-2,.5)--(4,.5);
 \draw[-] (-2,1)--(4,1);
 \draw[-] (-2,1.5)--(4,1.5);
  \draw[-] (-2,0)--(-2,1.5);
  \draw[-] (4,0)--(4,1.5);
  \draw[-] (2.5,.5)--(2.5,1.5);
  \draw[-] (.5,0)--(.5,1);
  \node at (2.5,1.7) {$y$};
   \node at (-2,1.7) {$0$};
   \node at (.5,1.7) {$z$};
   \node at (4,1.7) {$1$};
 \node at (-2.4,.25) {$L_3$};
 \node at (-2.4,.75) {$L_2$};
\node at (-2.4,1.25) {$L_1$};
  \node at (1,-0.3)   {$\mathfrak{C}=(L_1,L_2,L_3)$};   
  \node at (.25,1.25) {$P_{11}$};
  \node at(3.25,.75) {$P_{12}$};
  \node at(3.25,1.25) {$P_{31}$};
  \node at (-.75,.25){$P_{33}$};
  \node at (-.75,.75) {$P_{22}$};
  \node at (2.25,.25) {$P_{23}$};
  \node at (1.5,.75) {$P_{32}$};
  \node at (1,-.7) { Figure 5: Multi-allocation \( \mathcal{P}\) when $z\leq y$.};
  \end{tikzpicture}
\end{center}

\textbf{Subcase I:} In the case when $r=1$, we allocate the layered piece \( \mathcal{P}_2\) to agent 1 and construct a new cake $\mathfrak{C}^\ast=(L_1^\ast,L_2^\ast)$ where $L_1^\ast=L_1$ and $L_2^\ast=(L_3\cap[0,z])\bigcup (L_2\cap [z,1])$. The value of agents $2$ and $3$ on the new cake $\mathfrak{C}^\ast$ is now at least $\nicefrac{2}{3}$.
We now show that by applying the cut-and-choose procedure between agents $2$ and $3$ on the new cake $\mathfrak{C}$, we can achieve the required multi-allocation $\mathcal{P}^\ast$ for three agents. Due to Lemma $\ref{lem: switch}$, we obtain a point $x$ such that $V_2(LR(x,\mathfrak{C}^\ast))=V_2(RL(x,\mathfrak{C}^\ast))=\frac{V_2(\mathfrak{C}^\ast)}{2}> \nicefrac{1}{3}$. Out of the diagonal pieces, Agent $3$ selects a layered piece that he or she only weakly prefers. Without loss of generality, we assume that agent $3$ prefers $RL(x,\mathfrak{C}^\ast)$, and agent $2$ receives the remaining diagonal piece $LR(x,\mathfrak{C}^\ast)$.
According to the positions of $x$ and $z$, the required multi-allocation $\mathcal{P}^\ast$ should be different.
\newline 
When $x\leq z$, the required multi-allocation \(\mathcal{P}^\ast=(\mathcal{P}_1^\ast,\mathcal{P}_2^\ast,\mathcal{P}_3^\ast)\)  where \(\mathcal{P}_1^\ast=(\emptyset,L_2\cap [0,z],L_3\cap [z,1])\), \(\mathcal{P}_2^\ast=(L_1\cap [0,x],L_2\cap [z,1],L_3\cap [x,z])\) and \(\mathcal{P}_3^\ast=(L_1\cap [x,1],\emptyset,L_3\cap [0,x])\).
\newline
When $z\leq x $, the required multi-allocation \(\mathcal{P}^\ast=(\mathcal{P}_1^\ast,\mathcal{P}_2^\ast,\mathcal{P}_3^\ast)\) where \(\mathcal{P}_1^\ast=(\emptyset,L_2\cap [0,z],L_3\cap [z,1])\), \(\mathcal{P}_2^\ast=(L_1\cap [0,x],L_2\cap [x,1],\emptyset)\), and \(\mathcal{P}_3^\ast=(L_1\cap [x,1],L_2\cap [z,x],L_3\cap [0,z])\).
\begin{center}

     \begin{tikzpicture}
  \draw [thick,-] (-2,0)--(4,0);
 \draw[-] (-2,.5)--(4,.5);
 \draw[-] (-2,1)--(4,1);
 \draw[-] (-2,1.5)--(4,1.5);
  \draw[-] (-2,0)--(-2,1.5);
  \draw[-] (4,0)--(4,1.5);
  \draw[-] (-.5,1)--(-.5,1.5);
  \draw[-] (-.5,0)--(-.5,.5);
  \draw[-] (.5,0)--(.5,1);
  \node at (-.5,1.7) {$x$};
   \node at (-2,1.7) {$0$};
   \node at (.5,1.7) {$z$};
   \node at (4,1.7) {$1$};
 \node at (-2.4,.25) {$L_3$};
 \node at (-2.4,.75) {$L_2$};
\node at (-2.4,1.25) {$L_1$};
  \node at (1,-0.3)   {$\mathfrak{C}=(L_1,L_2,L_3)$};  \node at (-.75,.75) {$P^\ast_{12}$};
  \node at (2.25,.25) {$P^\ast_{13}$};
  \node at (-1.25,1.25) {$P^\ast_{21}$};
  \node at (2.25,.75)  {$P^\ast_{22}$};
  \node at (0,.25) {$P^\ast_{23}$};
  \node at (1.75,1.25) {$P^\ast_{31}$};
  \node at (-1.25,0.25) {$P^\ast_{33}$};
  \node at (1,-.7){Figure 6: 
    Multi-allocation \(\mathcal{P}^\ast\)
             when $x\leq z$.};
\end{tikzpicture}
    \begin{tikzpicture}
  \draw [thick,-] (-2,0)--(4,0);
 \draw[-] (-2,.5)--(4,.5);
 \draw[-] (-2,1)--(4,1);
 \draw[-] (-2,1.5)--(4,1.5);
  \draw[-] (-2,0)--(-2,1.5);
  \draw[-] (4,0)--(4,1.5);
  \draw[-] (2,.5)--(2,1.5);
  \draw[-] (.5,0)--(.5,1);
  \node at (2,1.7) {$x$};
   \node at (-2,1.7) {$0$};
   \node at (.5,1.7) {$z$};
   \node at (4,1.7) {$1$};
 \node at (-2.4,.25) {$L_3$};
 \node at (-2.4,.75) {$L_2$};
\node at (-2.4,1.25) {$L_1$};
  \node at (1,-0.3)   {$\mathfrak{C}=(L_1,L_2,L_3)$};  
   \node at (-.75,.75) {$P^\ast_{12}$};
  \node at (2.25,.25) {$P^\ast_{13}$};
  \node at (0,1.25) {$P^\ast_{21}$};
  \node at (3,.75) {$P^\ast_{22}$};
  \node at (3,1.25) {$P^\ast_{31}$};
  \node at (1.25,.75) {$P^\ast_{32}$};
  \node at (-1.25,0.25) {$P^\ast_{33}$};
  \node at (1,-.7){Figure 7: Multi-allocation \(\mathcal{P}^\ast\) when $z\leq x$.};
  
\end{tikzpicture}

\end{center}

\textbf{Subcase II:} In the scenario when $r=2$ or $3$, we assign the layered piece \(\mathcal{P}_1\) to agent 1, and we define a new cake $\mathfrak{C}^\ast=(L_1^\ast,L_2^\ast)$ where $L_1^\ast=L_2\cap [0,y]\bigcup L_1\cap [y,1]$ and
$L_2^\ast=L_3$. Now agents $2$ and $3$ have a value over the new cake $\mathfrak{C}^\ast$ of at least two-thirds. Similar to Subcase I, we obtain the required multi-allocation $\mathcal{P}^\ast$ that depends on the positions of $x$ and $y$, where $x$ is a switching point of agent $2$. \newline
When $x\leq y$, the required multi-allocation \(\mathcal{P}^\ast=(\mathcal{P}_1^\ast,\mathcal{P}_2^\ast,\mathcal{P}_3^\ast)\) where \(\mathcal{P}_1^\ast=\mathcal{P}_1\), \(\mathcal{P}_2^\ast =(\emptyset,L_2\cap [0,x],L_3\cap [x,1])\) and \((\mathcal{P}_3^\ast=(L_1\cap [y,1],L_2\cap [x,y],L_3\cap [0,x])\).
\newline
When $y\leq x$, the required multi-allocation \(\mathcal{P}\ast=(\mathcal{P}_1^\ast,\mathcal{P}_2^\ast,\mathcal{P}_3^\ast)\) where \(\mathcal{P}_1^\ast=\mathcal{P}_1\), \(\mathcal{P}_2^\ast=(L_1\cap [y,x],L_2\cap [0,y],L_3\cap [x,1])\) and \(\mathcal{P}_3^\ast=(L_1\cap [x,1],\emptyset,L_3\cap [0,x])\).

\begin{center}
     \begin{tikzpicture}
  \draw [thick,-] (-2,0)--(4,0);
 \draw[-] (-2,.5)--(4,.5);
 \draw[-] (-2,1)--(4,1);
 \draw[-] (-2,1.5)--(4,1.5);
  \draw[-] (-2,0)--(-2,1.5);
  \draw[-] (4,0)--(4,1.5);
  \draw[-] (2.5,.5)--(2.5,1.5);
  \draw[-] (0,0)--(0,1);
  \node at (2.5,1.7) {$y$};
   \node at (-2,1.7) {$0$};
   \node at (0,1.7) {$x$};
   \node at (4,1.7) {$1$};
 \node at (-2.4,.25) {$L_3$};
 \node at (-2.4,.75) {$L_2$};
\node at (-2.4,1.25) {$L_1$};
  \node at (1,-0.3)   {$\mathfrak{C}=(L_1,L_2,L_3)$}; 
  \node at (.25,1.25) {$P^\ast_{11}$};
  \node at(3.25,.75) {$P^\ast_{12}$};
  \node at (-1,.75){$P^\ast_{22}$};
  \node at (2,.25) {$P^\ast_{23}$};
  \node at (3.25,1.25) {$P^\ast_{31}$};
  \node at (1.25,.75) {$P^\ast_{32}$};
  \node at (-1,0.25) {$P^\ast_{33}$};
  \node at (1,-.7){ Figure 8: Multi-allocation \(\mathcal{P}^\ast\) when $x\leq y$.};
\end{tikzpicture}
\begin{tikzpicture}
  \draw [thick,-] (-2,0)--(4,0);
 \draw[-] (-2,.5)--(4,.5);
 \draw[-] (-2,1)--(4,1);
 \draw[-] (-2,1.5)--(4,1.5);
  \draw[-] (-2,0)--(-2,1.5);
  \draw[-] (4,0)--(4,1.5);
  \draw[-] (2.5,.5)--(2.5,1.5);
  \draw[-] (3.2,0)--(3.2,.5);
  \draw[-] (3.2,1)--(3.2,1.5);
  \node at (2.5,1.7) {$y$};
   \node at (-2,1.7) {$0$};
   \node at (3.2,1.7) {$x$};
   \node at (4,1.7) {$1$};
 \node at (-2.4,.25) {$L_3$};
 \node at (-2.4,.75) {$L_2$};
\node at (-2.4,1.25) {$L_1$};
  \node at (1,-0.3)   {$\mathfrak{C}=(L_1,L_2,L_3)$}; 
  \node at (.25,1.25) {$P^\ast_{11}$};
  \node at (3.25,.75) {$P^\ast_{12}$};
  \node at (2.85,1.25) {$P^\ast_{21}$};
  \node at (.25,.75) {$P^\ast_{22}$};
  \node at (3.6,.25) {$P^\ast_{23}$};
  \node at (3.6,1.25) {$P^\ast_{31}$};
  \node at (1.6,.25) {$P^\ast_{33}$};
  \node at (1,-.7){ Figure 9: Multi-allocation \(\mathcal{P}^\ast\) when $y\leq x$.};
\end{tikzpicture}
\end{center}

\textbf{Case 2:} When $y\leq z$, \(\mathcal{P}_1=(L_1\cap [0,y],L_2\cap [y,1],\emptyset)\), \(\mathcal{P}_2=(L_1\cap [y,z],L_2\cap [0,y],L_3\cap [z,1])\) and \(\mathcal{P}_3=(L_1\cap [z,1], \emptyset,L_3\cap [0,z])\) are the equivalent layered pieces. In the case in which two distinct layered pieces \(\mathcal{P}_p\) and \((\mathcal{P}_q\) are obtained, such that \(V_2(\mathcal{P}_p)\geq \frac{1}{3}\) and \(V_3(\mathcal{P}_q)\geq \frac{1}{3}\), assign \(\mathcal{P}_p\) to agent $2$, \(\mathcal{P}_q\) to agent $3$, and the last layered piece to agent $1$, where $p\neq q \in \{1,2,3\}$. Thus, we obtain a feasible and contiguous proportional multi-allocation, $\mathcal{P}^\ast=\mathcal{P}$. Otherwise, for agents $2$ and $3$, there is only one unique layered piece \(\mathcal{P}_r\) with a value greater than $\nicefrac{1}{3}$, where $r\in \{1,2,3\}$.
\begin{center}
     \begin{tikzpicture}
  \draw [thick,-] (-2,0)--(4,0);
 \draw[-] (-2,.5)--(4,.5);
 \draw[-] (-2,1)--(4,1);
 \draw[-] (-2,1.5)--(4,1.5);
  \draw[-] (-2,0)--(-2,1.5);
  \draw[-] (4,0)--(4,1.5);
  \draw[-] (2.5,.5)--(2.5,1.5);
  \draw[-] (3,0)--(3,.5);
  \draw[-] (3,1)--(3,1.5);
  \node at (2.5,1.7) {$y$};
   \node at (-2,1.7) {$0$};
   \node at (3,1.7) {$z$};
   \node at (4,1.7) {$1$};
 \node at (-2.4,.25) {$L_3$};
 \node at (-2.4,.75) {$L_2$};
\node at (-2.4,1.25) {$L_1$};
  \node at (1,-0.3)   {$\mathfrak{C}=(L_1,L_2,L_3)$}; 
  \node at (.25,1.25) {$P_{11}$};
  \node at (3.25,.75) {$P_{12}$};
  \node at (2.75,1.25) {$P_{21}$};
  \node at (.25,.75) {$P_{22}$};
  \node at (3.5,.25) {$P_{23}$};
  \node at (3.5,1.25) {$P_{31}$};
  \node at (.5,.25) {$P_{33}$};
  \node at (1,-.7){Figure 10: Multi-allocation $\mathcal{P}$ when $y\leq z$.};
\end{tikzpicture}
\end{center}

\textbf{Subcase I:} We define a new cake $\mathfrak{C}^\ast=(L_1^\ast,L_2^\ast)$ and assign the layered piece \(\mathcal{P}_3\) to agent $1$ when $r=1$, where $L_1^\ast=(L_1\cap [0,z])\bigcup (L_3 \cap [z,1])$ and $L_2^\ast=L_2$. Another two agents, $2$ and $3$, have valuations at least $\nicefrac{2}{3}$ over the new cake $\mathfrak{C}^\ast$. We now show that we can obtain the required multi-allocation $\mathcal{P}^\ast$ for three agents by using the cut-and-choose procedure among agents $2$ and $3$ over the new cake $\mathcal{C}^\ast$. In accordance with Lemma $\ref{lem: switch}$, we must get a point $x$ that satisfies the condition $V_2(LR(x,\mathfrak{C}^\ast))=V_2(RL(x,\mathfrak{C}^\ast))=\frac{V_2(\mathfrak{C}^\ast)}{2}>\frac{1}{3}$. Without loss of generality, we assume that agent $3$ chooses $RL(x,\mathfrak{C}^\ast)$ and agent $2$ obtains $LR(x,\mathfrak{C}^\ast)$. The required multi-allocation $\mathcal{P}^\ast$ must differ depending on where $x$ and $z$ are.\newline 
When $x\leq z,$ the required multi-allocation  \(\mathcal{P}^\ast=(\mathcal{P}_1^\ast,\mathcal{P}_2^\ast,\mathcal{P}_3^\ast)\) where \(\mathcal{P}_1^\ast=(L_1\cap [z,1],\emptyset,L_3\cap [o,z])\), \(\mathcal{P}_2^\ast=(L_1\cap [0,x],L_2\cap [x,1],\emptyset)\) and \(\mathcal{P}_3^\ast=(L_1\cap [x,z],L_2\cap [o,x],L_3\cap[z,1])\).\newline
When $z\leq x$, the layered pieces of required multi-allocation $\mathcal{P}^\ast$ are \(\mathcal{P}_1^\ast=(L_1\cap [z,1],\emptyset,L_3\cap [0,z])\), \(\mathcal{P}_2^\ast=(L_1\cap [0,z],L_2\cap [x,1],L_3\cap [z,x])\) and \(\mathcal{P}_3^\ast=(\emptyset,L_2\cap [0,x],L_3\cap [x,1])\).

\begin{center}
     \begin{tikzpicture}
  \draw [thick,-] (-2,0)--(4,0);
 \draw[-] (-2,.5)--(4,.5);
 \draw[-] (-2,1)--(4,1);
 \draw[-] (-2,1.5)--(4,1.5);
  \draw[-] (-2,0)--(-2,1.5);
  \draw[-] (4,0)--(4,1.5);
  \draw[-] (1.5,.5)--(1.5,1.5);
  \draw[-] (3,0)--(3,.5);
  \draw[-] (3,1)--(3,1.5);
  \node at (1.5,1.7) {$x$};
   \node at (-2,1.7) {$0$};
   \node at (3,1.7) {$z$};
   \node at (4,1.7) {$1$};
 \node at (-2.4,.25) {$L_3$};
 \node at (-2.4,.75) {$L_2$};
\node at (-2.4,1.25) {$L_1$};
  \node at (1,-0.3)   {$\mathfrak{C}=(L_1,L_2,L_3)$};
  \node at (3.5,1.25) {$P^\ast_{11}$};
  \node at (.5,.25) {$P^\ast_{13}$};
  \node at (-.25,1.25){$P^\ast_{21}$};
  \node at (2.75,.75) {$P^\ast_{22}$};
  \node at (-.25,.75) {$P^\ast_{32}$};
  \node at (3.5,.25) {$P^\ast_{33}$};
  \node at (2.25,1.25) {$P^\ast_{31}$};
  \node at (1,-.7){Figure 11: Multi-allocation \(\mathcal{P}^\ast\)when $x\leq z$.};
\end{tikzpicture}
\begin{tikzpicture}
  \draw [thick,-] (-2,0)--(4,0);
 \draw[-] (-2,.5)--(4,.5);
 \draw[-] (-2,1)--(4,1);
 \draw[-] (-2,1.5)--(4,1.5);
  \draw[-] (-2,0)--(-2,1.5);
  \draw[-] (4,0)--(4,1.5);
  \draw[-] (3.5,0)--(3.5,1);
  \draw[-] (3,0)--(3,.5);
  \draw[-] (3,1)--(3,1.5);
  \node at (3.5,1.7) {$x$};
   \node at (-2,1.7) {$0$};
   \node at (3,1.7) {$z$};
   \node at (4,1.7) {$1$};
 \node at (-2.4,.25) {$L_3$};
 \node at (-2.4,.75) {$L_2$};
\node at (-2.4,1.25) {$L_1$};
  \node at (1,-0.3)   {$\mathfrak{C}=(L_1,L_2,L_3)$};   
  \node at (3.5,1.25) {$P^\ast_{11}$};
  \node at (.5,.25) {$P^\ast_{13}$};
  \node at (.75,1.25){$P^\ast_{21}$};
  \node at (3.25,.25){$P^\ast_{23}$};
  \node at (3.75,.75) {$P^\ast_{22}$};
  \node at (3.75,.25){$P^\ast_{33}$};
  \node at (1.75,.75){$P^\ast_{32}$};
  \node at (1,-.7){Figure 12: Multi-allocation \(\mathcal{P}^\ast\) when $z\leq x$.};

\end{tikzpicture}

\end{center}

\textbf{Subcase II:} When $r=2$ or $3$, we allocate the layered piece \(\mathcal{P}_1\) to agent $1$ and define a new cake $\mathfrak{C}^\ast=(L_1^\ast,L_2^\ast)$ where $L_1^\ast=(L_1\cap [0,y] \bigcup L_2\cap [y,1])$ and $L_2^\ast=L_3$. The value of agents $2$ and $3$ on the new cake $\mathfrak{C}^\ast$ is now at least $\frac{2}{3}$. Similar to Subcase I, We get the required multi-allocation $\mathcal{P}^\ast$, which depends on the locations of $X$ and $y$, where $x$ is a switching point of agent $2$.\newline
When $x\leq y$, the required multi-allocation \(\mathcal{P}^\ast=(\mathcal{P}_1^\ast,\mathcal{P}_2^\ast,\mathcal{P}_3^\ast)\) where \(\mathcal{P}_1^\ast=(L_1\cap [0,y],L_2\cap [y,1],\emptyset)\), \(\mathcal{P}_2^\ast=(\emptyset,L_2\cap [0,x],L_3\cap [x,1])\) and \(\mathcal{P}_3^\ast=(L_1\cap [y,1],L_2\cap [x,y],L_3 \cap[0,x])\). \newline
When $y\leq x,$ the layered pieces of the required multi-layered cake are \(\mathcal{P}_1^\ast=(L_1\cap [0,y],L_2\cap [y,1],\emptyset)\), \(\mathcal{P}_2^\ast=(L_1\cap [y,x],L_2\cap [0,y],L_3\cap [x,1])\) and \(\mathcal{P}_3^\ast=(L_1\cap [x,1],\emptyset,L_3\cap [0,x])\).
\begin{center}
   \begin{tikzpicture}
  \draw [thick,-] (-2,0)--(4,0);
 \draw[-] (-2,.5)--(4,.5);
 \draw[-] (-2,1)--(4,1);
 \draw[-] (-2,1.5)--(4,1.5);
  \draw[-] (-2,0)--(-2,1.5);
  \draw[-] (4,0)--(4,1.5);
  \draw[-] (2.5,.5)--(2.5,1.5);
  \draw[-] (1,0)--(1,1);
  \node at (2.5,1.7) {$y$};
   \node at (-2,1.7) {$0$};
   \node at (1,1.7) {$x$};
   \node at (4,1.7) {$1$};
 \node at (-2.4,.25) {$L_3$};
 \node at (-2.4,.75) {$L_2$};
\node at (-2.4,1.25) {$L_1$};
  \node at (1,-0.3)   {$\mathfrak{C}=(L_1,L_2,L_3)$}; 
  \node at (.25,1.25){$P^\ast_{11}$};
  \node at (3.25,.75){$P^\ast_{12}$};
  \node at (-.5,.75){$P^\ast_{22}$};
  \node at (2.5,.25){$P^\ast_{23}$};
  \node at (3.25,1.25){$P^\ast_{31}$};
  \node at (1.75,.75){$P^\ast_{32}$};
  \node at (-.5,.25) {$P^\ast_{33}$};
  \node at (1,-.7){Figure 13: Multi-allocation \(\mathcal{P}^\ast\) when $x\leq y$.};
\end{tikzpicture}
\begin{tikzpicture}
  \draw [thick,-] (-2,0)--(4,0);
 \draw[-] (-2,.5)--(4,.5);
 \draw[-] (-2,1)--(4,1);
 \draw[-] (-2,1.5)--(4,1.5);
  \draw[-] (-2,0)--(-2,1.5);
  \draw[-] (4,0)--(4,1.5);
  \draw[-] (2.5,.5)--(2.5,1.5);
  \draw[-] (3.2,0)--(3.2,.5);
  \draw[-] (3.2,1)--(3.2,1.5);
  \node at (2.5,1.7) {$y$};
   \node at (-2,1.7) {$0$};
   \node at (3.2,1.7) {$x$};
   \node at (4,1.7) {$1$};
 \node at (-2.4,.25) {$L_3$};
 \node at (-2.4,.75) {$L_2$};
\node at (-2.4,1.25) {$L_1$};
  \node at (1,-0.3)   {$\mathfrak{C}=(L_1,L_2,L_3)$}; 
  \node at (.25,1.25){$P^\ast_{11}$};
  \node at (3.25,.75){$P^\ast_{12}$};
  \node at (.25,.75){$P^\ast_{22}$};
  \node at (2.85,1.25){$P^\ast_{21}$};
  \node at (3.6,.25){$P^\ast_{23}$};
  \node at (3.6,1.25){$P^\ast_{31}$};
  \node at (.6,.25){$P^\ast_{33}$};
  \node at (1,-.7){Figure 14: Multi-allocation \(\mathcal{P}^\ast\) when $y\leq x$.};
\end{tikzpicture}
\end{center}
In either scenario, we obtain a proportional multi-allocation \(\mathcal{P}^\ast\) that satisfies the feasibility and contiguity criteria.
\end{proof}
Now we are ready to give computational procedure to find a proportional multi-allocation that satisfies feasibility and contiguity conditions when $n\geq 3$.
In these computational procedure, we recall the Theorem $\ref{2}$. In Theorem $\ref{2}$, Hosseini et al. \cite{20Hosseini2020FairDO} give the computational procedure to find a feasible and contiguous multi-allocation when $n\geq m$, where $m=2^a$ for $a\in \mathbb{Z}_+$.
\begin{theorem}\label{2}
A proportional complete multi-allocation that is feasible and contiguous exists, for any number $m$ of layers and any number $n\geq m$ of agents, where $m=2^a$ for some $a\in \mathbb{Z}_+$.
\end{theorem}
In Theorem \ref{thm: prograt3}, we show that feasible and contiguous proportional allocation can be computed efficiently when $n\geq m$ and $m=3$.
\begin{theorem}\label{thm: prograt3}
A proportional complete multi-allocation that is feasible and contiguous exists for three layers and any number $n\geq 3$ agents.    
\end{theorem}
\begin{proof}
    We show the computation of the required multi-allocation in two scenarios where $n$ is even or odd and $n\geq 4$. Theorem 3 implies the proof to take into account $n=3$.\newline
\textbf{Case 1:} Suppose that $n$ is even  and is of the form $n=2k$, where $k\in \mathbb{Z}_+\backslash \{1\}$. Because there are an odd number of layers, we are unable to directly apply the majority switching point property. After treating the first and second layers of the cake $\mathfrak{C}$, aka $L_1$ and $L_2$, as a single layer of a new cake $\mathfrak{C}^\prime=(L_1^\prime,L_2^\prime)$, we apply this property, where 
 $L_1^\prime=L_1\cup L_2$, $L_2^\prime=L_3$ and $\mathfrak{C}=(L_1,L_2,L_3)$. The non-negative integrable density function of an agent $i$ over the layer $L^\prime_1$ is $v_{i1} + v_{i2}$, where $v_{il}$ is the density function of the agent $i$ over the layer $L_l$ for $l=1,2$. $L^\prime_1 =L_1\cup L_2$ denotes that the layer $L_1$ and layer $L_2$ are mutually overlapping.

We obtain a majority switching point $x\in [0,1]$ due to the fact that the cake $\mathfrak{C}^\prime$ has an even number of layers. We set $\mathcal{I}_1$ to be $LR(x,\mathfrak{C}^\prime)$ and $\mathcal{I}_2$ to be $RL(x,\mathfrak{C}^\prime)$. Due to the fact that $n$ is even and the notion of a majority switching point, we may divide the set of agents
 \(\mathcal{N}\) into \(\mathcal{N}_1\) and \(\mathcal{N}_2\), such that \(|\mathcal{N}_1|={|\mathcal{N}_2|}=\frac{|\mathcal{N}|}{2}\), where \(\mathcal{N}_1\) is the set of agents who weakly prefer $\mathcal{I}_1$ to $\mathcal{I}_2$ and \(\mathcal{N}_2\)  is the set of agents who weakly prefer $\mathcal{I}_2$ to $\mathcal{I}_1$.

We have $\mathcal{I}_1=(L_1\cup L_2)\cap [0,x] \bigcup L_3\cap[x,1]$ and $\mathcal{I}_2= L_3\cap [0,x] \bigcup (L_1\cup L_2)\cap [x,1]$. 
We now take into account two other 2-layered cakes, $\mathfrak{C}^1$ and $\mathfrak{C}^{2}$, which are obtained by merging $\mathcal{I}_1$ and $\mathcal{I}_2$, respectively, where $\mathfrak{C}^1=(L_1\cap[0,x],L_2\cap[0,x]\bigcup L_3\cap[x,1])$ and $\mathfrak{C}^{2}=(L_1\cap[x,1],L_3\cap[0,x]\bigcup L_2\cap[x,1])$. As a result, each agent $i\in\mathcal{N}_j$ on the cake $\mathfrak{C}^j$, where $j = 1$ and $2$, has a value of at least $V_i(I_j)=\frac{V_i(LR(x,\mathfrak{C}))}{2}$. Thus, we get a proportional multi-allocation feasible and contiguous for the set of agents $\mathcal{N}_j$ on the cake $\mathcal{C}^j$, where $j=1$ and $2$, since $|\mathcal{N}_j|=k\geq 2$ and Theorem $\ref{2}$ imply that. As a result, merging the two multi-allocations yields a contiguous, feasible, and complete multi-allocation that guarantees that each agent receives a proportional share.
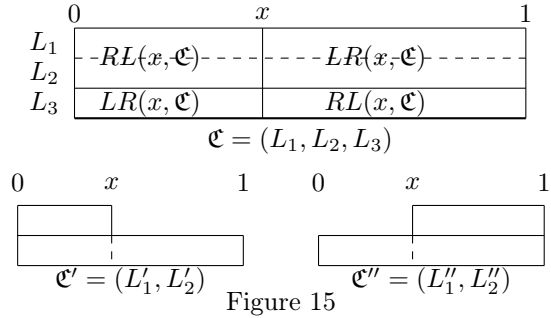
\begin{center}
     
 \begin{tikzpicture}
  \draw [thick,-] (-2,0)--(4,0);
 \draw[-] (-2,.4)--(4,.4);
 \draw[dashed] (-2,.8)--(4,.8);
 \draw[-] (-2,1.2)--(4,1.2);
  \draw[-] (-2,0)--(-2,1.2);
  \draw[-] (4,0)--(4,1.2);
  \draw[-] (.5,0)--(.5,1.2);
   \node at (.5,1.4) {$x$};
   \node at (-2,1.4) {$0$};
   \node at (4,1.4) {$1$};
   \node at (-1,.8) {$RL(x,\mathfrak{C})$};
   \node at (2,.8) {$LR(x,\mathfrak{C})$};
    \node at (-1,.2) {$LR(x,\mathfrak{C})$};
     \node at (2,.2) {$RL(x,\mathfrak{C})$};
     \node at (-2.4,.2) {$L_3$};
     \node at (-2.4,.6) {$L_2$};
     \node at (-2.4,1) {$L_1$};
     \node at (1,-0.3)   {$\mathfrak{C}=(L_1,L_2,L_3)$};
\end{tikzpicture}

\begin{tikzpicture}
\draw[-](0,0)--(1.25,0);
\draw[-](0,-.4)--(3,-.4);
\draw[-](0,-.8)--(3,-.8);
\draw[-](0,0)--(0,-.8);
\draw[-](3,-.4)--(3,-.8);
\draw[-](1.25,0)--(1.25,-.4);
\draw[dashed](1.25,-.4)--(1.25,-.8);
\node at (0,.3) {$0$};
\node at (1.25,.3){$x$};
\node at (3,.3) {$1$};
\node at (1.5,-1) {$\mathfrak{C}^\prime=(L_1^\prime,L_2^\prime)$};

 \draw[-] (5.25,0)--(7,0);
 \draw[-](4,-.4)--(7,-.4);
 \draw[-] (4,-.8)--(7,-.8);
 \draw[-](7,0)--(7,-.8);
 \draw[-](4,-.4)--(4,-.8);
 \draw[-](5.25,0)--(5.25,-.4);
 \draw[dashed](5.25,-.4)--(5.25,-.8);
  \node at (4,0.3) {$0$};
  \node at (5.25,0.3) {$x$};
  \node at (7,0.3) {$1$};
  \node at (5.5,-1) {$\mathfrak{C}^{\prime\prime}=(L_1^{\prime\prime},L_2^{\prime\prime})$};
  \node at (3.5,-1.3) {Figure 15};
 \end{tikzpicture}

\end{center}
\textbf{Case 2:} 
In the case when $n$ is odd, it can be expressed as $n=2k+1$, where $k\in \mathbb{Z}_+\backslash \{1\}$. In this case, our aim is to reduce this case to case 1. Without loss of generality, we assume that there are some agents such that each has a value is at least $\frac{1}{n}$ on the top layer cake $L_1$.  Then, we ask each agent $a_i$ to place a mark at point $y_i$ on cake layer $L_1$ so that the value of the piece $Y=L_1\cap [0,y_i]$ is equal to $\frac{1}{n}$ and allocate the piece $Y=L_1\cap [0,y]$ to agent $p$ where $y=y_p=$min$\{y_1,y_2,y_3,.....,y_n\}$. In order to decide how to share the cake among $n$ agents, we reduce it to an instance 
 ($\mathcal{N}\backslash\{p\}$, $(L_l^\prime)_{l\in \mathcal{M}}$, $(V_i)_{i\in \mathcal{N}^\prime\backslash\{p\}}$) where $L_1^{\prime}=L_1\backslash Y$ and $L_l^{\prime}=L_l$ for $l\neq 1$. Each agent $i\in \mathcal{N}$ has a value of at least $\frac{n-1}{n}V_i(\mathfrak{C})=\frac{2k}{2k+1}V_i(\mathfrak{C})$ on the remaining cake. Due to Case $1$, we obtain a proportional multi-allocation 
 \(\mathcal{P}\) that is contiguous and feasible, for instance ($\mathcal{N}\backslash\{p\}$, $(L_l^\prime)_{l\in \mathcal{M}}$, $(V_i)_{i\in \mathcal{N}\backslash\{p\}}$). Together with the allocated piece $L_1\cap [0,y]$ and the multi-allocation \(\mathcal{P}\), we get a proportional multi-allocation \(\mathcal{P}^\ast\) that is feasible and contiguous. Pictures are shown in figure 16.
 \begin{center}
      \begin{tikzpicture}
  \draw [thick,-] (-2,0)--(4,0);
 \draw[-] (-2,.4)--(4,.4);
 \draw[dashed] (-2,.8)--(4,.8);
 \draw[-] (-2,1.2)--(4,1.2);
  \draw[-] (-2,0)--(-2,1.2);
  \draw[-] (4,0)--(4,1.2);
  \draw[-] (.5,0)--(.5,1.2);
  \draw[-] (-1.7,.8)--(-1.7,1.2);
  \node at (-1.7,1.4){$y$};
  \node at (.5,1.4) {$x$};
   \node at (-2,1.4) {$0$};
   \node at (4,1.4) {$1$};
   \node at (-1,.8) {$RL(x,\mathfrak{C})$};
   \node at (2,.8) {$LR(x,\mathfrak{C})$};
    \node at (-1,.2) {$LR(x,\mathfrak{C})$};
     \node at (2,.2) {$RL(x,\mathfrak{C})$};
     \node at (-2.4,.2) {$L_3$};
     \node at (-2.4,.6) {$L_2$};
     \node at (-2.4,1) {$L_1$};
     \node at (1,-0.3)   {$\mathfrak{C}=(L_1,L_2,L_3)$};
\end{tikzpicture}

 \begin{tikzpicture}
\draw[-](.18,0)--(1.25,0);
\draw[-](0,-.4)--(3,-.4);
\draw[-](0,-.8)--(3,-.8);
\draw[-](0,-.4)--(0,-.8);
\draw[-](3,-.4)--(3,-.8);
\draw[-](1.25,0)--(1.25,-.4);
\draw[-](0.18,0)--(0.18,-.4);
\draw[dashed](1.25,-.4)--(1.25,-.8);
\node at (.18,.3){$y$};
\node at (0,.3) {$0$};
\node at (1.25,.3){$x$};
\node at (3,.3) {$1$};
\node at (1.5,-1) {$\mathfrak{C}^\prime=(L_1^\prime,L_2^\prime)$};

 \draw[-] (5.25,0)--(7,0);
 \draw[-](4,-.4)--(7,-.4);
 \draw[-] (4,-.8)--(7,-.8);
 \draw[-](7,0)--(7,-.8);
 \draw[-](4,-.4)--(4,-.8);
 \draw[-](5.25,0)--(5.25,-.4);
 \draw[dashed](5.25,-.4)--(5.25,-.8);
  \node at (4,0.3) {$0$};
  \node at (5.25,0.3) {$x$};
  \node at (7,0.3) {$1$};
  \node at (5.5,-1) {$\mathfrak{C}^{\prime\prime}=(L_1^{\prime\prime},L_2^{\prime\prime})$};
  \node at (3.5,-1.3){Figure 16};
 \end{tikzpicture}
 \end{center}
 \end{proof}
We will use the following lemma to computing proportional allocation for any number $n\geq 2^a3$ of agents and $2^a3$ layers, where $a$ is any positive integer.
\begin{lemma}\label{lem: nonoverlapping}
Let $\mathfrak{C}$ be a $2m$-layered cake and $x \in [0, 1]$.
Suppose that $\mathfrak{C}^\prime$  is a $m$-layered cake obtained by merging $LR(x,\mathfrak{C})$ or $RL(x,\mathfrak{C})$. Then, each non-overlapping contiguous layered piece of $\mathfrak{C}^\prime$ is a non-overlapping contiguous layered piece of the original cake $\mathfrak{C}$ \cite{20Hosseini2020FairDO}.  
\end{lemma}

In Theorem \ref{thm: 23}, we show that feasible and contiguous proportional allocation can be computed efficiently when $m=n$, where the number $m$ of layers is of the form $2^a3$ and $a$ is a positive integer.
\begin{theorem}\label{thm: 23}
 If the number of agents and the number of layers are equal and is of the form $2^a3$ where $a\in \mathbb{Z}_+$, then we can compute a proportional multi-allocation that is contiguous and feasible.   
\end{theorem}
\begin{proof}
Without loss of generality, we assume that every layered cake $L_l$ is a replica of the unit interval $[0,1]$.
We design the following recursive algorithm $\mathcal{A}l$, which accepts a \(|\mathcal{M}^\prime|\)-layered cake $\mathfrak{C}^\prime$ together with a subset \(\mathcal{N}^\prime\) of agents with \(|\mathcal{N}^\prime|\geq 2\) and a valuation profile \((V_i)_{i\in \mathcal{N}^\prime }\) and yields a proportional complete multi-allocation of the cake to the agents that is feasible.\newline
Now consider the case when $m = n = 2^a 3$ for some integers $a\geq 1$. The algorithm $\mathcal{A}l$ searches a majority switching point $x$ over the cake $\mathfrak{C}^\prime$. We let \(\mathcal{I}_1=LR(x,\mathfrak{C}^\prime)\) and \(\mathcal{I}_2=RL(x,\mathfrak{C}^\prime)\). Due to the fact that $n$ is even, and by the definition of a majority switching point, we can split the set of agents \(\mathcal{N}^\prime\) into \(\mathcal{N}_1\) and \(\mathcal{N}_2\) such that \(|\mathcal{N}_1|=|\mathcal{N}_2|=\nicefrac{|\mathcal{N}^\prime|}{2}\), where \(\mathcal{N}_1\) is the set of agents who weakly prefer \(\mathcal{I}_1\) to \(\mathcal{I}_2\) and \(\mathcal{N}_2\) be the set of agents who weakly prefer \(\mathcal{I}_2\) to \(\mathcal{I}_1\). Now we run the algorithm \(\mathcal{A} l\) on the cake $\mathfrak{C}^i$ with the set of agents \(\mathcal{N}_i\) for each $i=1,2,$ respectively, where $\mathfrak{C}^i$ is the merge of \(\mathcal{I}_i\) for $i=1,2$. The complete multi-allocation \(\mathcal{P}=(\mathcal{P}_1,\mathcal{P}_2,...,\mathcal{P}_n)\) returned by $\mathcal{A}l$ attains proportionality in addition to feasibility and contiguity, as we will show via induction on the exponential $a$.\newline
In the case when $a=1$, $m=n=6$  and \(|\mathcal{N}_1|=|\mathcal{N}_2|=3\) occur from the scenario. Thus each $\mathfrak{C}^i$ is $3$-layered cake due to the fact that $\mathfrak{C}^\prime$ is $6$-layered cake and $\mathfrak{C}^i$ is obtained from the merge of a diagonal piece $\mathcal{I}_i$ of $\mathfrak{C}^\prime$ where  \(\mathcal{I}_1=LR(x,\mathfrak{C}^\prime)\) and \(\mathcal{I}_2=RL(x,\mathfrak{C}^\prime)\). Theorem \ref{thm: pro3} yields a feasible and contiguous multi-allocation $\mathcal{P}^i$ over $\mathfrak{C}^i$ for the set of agents \(\mathcal{N}_i\) that satisfies proportionality condition. Each agent $p\in \mathcal{N}_i$ receives a value of at least $\frac{V_p(\mathcal{I}_i)}{|\mathcal{N}_i|}\geq \frac{V_p(\mathfrak{C}^\prime)}{2|\mathcal{N}_i|}=\frac{1}{2P^b}=\frac{1}{n}$. Therefore, as stated in Lemma \ref{lem: nonoverlapping}, merging both multi-allocations can result in a contiguous and feasible complete multi-allocation that ensures each agent a proportional share.
\newline 
Assume that the claim is true for $m =n=2^a3$ with $1\leq a\leq k$; we will prove it for $a=k+1$ and $b\geq 1$. If $a=k+1$, then $m=n=2^{k+1}3$. Assume that the algorithm $\mathcal{A}l$ partitions the input cake $\mathfrak{C}^\prime$ into 
\(\mathcal{I}_1=LR(x,\mathfrak{C}^\prime)\) and 
\(\mathcal{I}_2=RL(x,\mathfrak{C}^\prime)\) making use of the majority switching point x. Suppose that the agents are divided into two groups (\( \mathcal{N}_1 \), \( \mathcal{N}_2 \)) where \(\mathcal{N}_1\) is the set of agents who weakly prefer \(\mathcal{I}_1\) to \(\mathcal{I}_2\) and \(\mathcal{N}_2\) be the set of agents who weakly prefer \(\mathcal{I}_2\) to \(\mathcal{I}_1\). Notice that the set of agents \(\mathcal{N}_1\) is the set of agents who weakly prefer 
\(\mathcal{I}_1\) to \(\mathcal{I}_2\), which results in $V_i(\mathcal{I}_1)\geq \frac{V_i(\mathfrak{C}^\prime)}{2}$ for all 
$i\in \mathcal{N}_1$. Similarly, $V_i(\mathcal{I}_2)\geq \frac{V_i(\mathfrak{C}^\prime)}{2}$ for all 
$i\in \mathcal{N}_2$. Thus, by the induction hypothesis, each agent $i$ has value at least $\frac{V_i(\mathfrak{C}^\prime)}{|\mathcal{N}^\prime|}$ for its assigned layered piece $\mathcal{P}_i$. According to the induction hypothesis, the algorithm $\mathcal{A}l$ generates a feasible and contiguous multi-allocation for each merge. Lemma \ref{lem: nonoverlapping} implies that every non-overlapping and contiguous layered piece of the merge of $\mathcal{I}_i$ is also a non-overlapping and contiguous layered piece of the original cake. Thus, the algorithm gives a proportional complete multi-allocation that is contiguous and feasible as an output.
\end{proof}

We will further develop the aforementioned theorem to the situation in which there are strictly more agents than layers. When $n>m$, it comes to intuitive understanding that there is at least one layer whose sub-piece may be "safely" assigned to a particular agent without leaving the non-overlapping condition. In the following theorem, we show that feasible and contiguous proportional allocation can be computed efficiently when $m<n$, where the number $m$ of layers is of the form $2^a3$ and $a$ is a positive integer number.
\begin{theorem}\label{thm: grat23}
    There exists a proportional multi-allocation over $m$-layered cake for any number $n\geq m$ of agents that is feasible and contiguous, where $m$ is of form $2^a3$ for all $a\in \mathbb{Z}_+$.
\end{theorem} 
\begin{proof}
We design the following recursive algorithm $\mathcal{A}l$, which accepts a \(|\mathcal{M}^\prime|\)-layered cake $\mathfrak{C}^\prime$ together with a subset \(\mathcal{N}^\prime\) of agents with \(|\mathcal{N}^\prime|\geq 2\) and a valuation profile \((V_i)_{i\in \mathcal{N}^\prime }\) and yields a proportional complete multi-allocation of the cake to the agents that is feasible.\newline
When $n=m$, we apply the algorithm outlined in Theorem \ref{thm: 23}'s proof. In the case when $n>m$, $m=2^a3$ for some integers $a$, the algorithm finds a layer $L_l$ where at least some agents have values of at least $\frac{1}{n}$ on $L_l$. Without loss of generality, we assume that $l=1$. The algorithm then instructs each agent $a_i$ to place a mark at point $y_i$ on cake layer $L_1$  so that the value of the piece $Y=L_1\cap [0,y_i]$ is equal to $\frac{1}{n}$ and allocate the piece $Y=L_1\cap [0,y]$ to agent $p$ where $y=y_p=$min$\{y_1,y_2,y_3,.....,y_n\}$. In order to decide how to share the remaining cake, we perform the algorithm $\mathcal{A} l$ to the reduced instance ($\mathcal{N}^\prime\backslash\{p\}$, $(L_l^\prime)_{l\in \mathcal{M}}$, $(V_i)_{i\in \mathcal{N}^\prime\backslash\{p\}}$) where $L_1^{\prime}=L_1\backslash Y$ and $L_l^{\prime}=L_l$ for $l\neq 1$.
We will show via induction on $\mathcal{N}^\prime$ that the complete multi-allocation $\mathcal{P}=(\mathcal{P}_1,\mathcal{P}_2,.....,\mathcal{P}_n)$ returned by $\mathcal{A} l$ follows proportionality as well as feasibility and contiguity. Due to Theorem \ref{thm: 23}, this is true for $|\mathcal{N}^\prime|=m$.
For $m\leq |\mathcal{N}^\prime|\leq k-1$, let's assume that the claim is true.
We will now show that the claim also holds for $|\mathcal{N}^\prime|=k$. Assume that agent $a$ gets the contiguous piece $Y$. Clearly, agent $a$ receives a proportional value under $\mathcal{P}$.
Observe that for the remaining cake, all remaining agents have a value of at least $\frac{|\mathcal{N}^\prime|-1}{|\mathcal{N}^\prime|}V_i(\mathfrak{C}^\prime)$. Thus, by the induction hypothesis, each agent $i\neq a$ has value at least $\frac{V_i(\mathfrak{C}^\prime)}{|\mathcal{N}^\prime|}$ for its assigned layered piece $\mathcal{P}_i$. The induction hypothesis makes it obvious that $\mathcal{P}$ is both feasible and contiguous. The proof has concluded with the above.  
\end{proof}

\section{Discussion}\label{sec:discuusion}
We study the problem of multi-layered cake cutting, where we divide multi-layered cake among a set of agents under two constraints, feasibility and contiguity. In section \ref{sec:exact}, we propose a new computational model. Then, we show the existence of an exact feasible multi-allocation for two agents and two layers using the new computational model. In section \ref{sec:proportional}, we show that a proportional multi-allocation can be computed for three layers and any number of agents greater than three using the cut-and-eval queries. We also show a technique for computing a proportional allocations for any number $n\geq 2^a3$ of agents and $2^a3$ layers, where $a$ is any positive integer.
Igarashi and Meunier \cite{25igarashi2021envy} show the existence of a feasible and contiguous proportional multi-allocation for any number $m$ of layers and any number $n\geq m$ of agents. But our results are based on the computation of proportional multi-allocation for various agents and layers.

There are several directions that are open for future research. 
    
    $\bullet$ \textbf{Query complexity of envy-free multi-allocation :}
    In the multi-layered cake-cutting problem, the query complexity of finding a feasible multi-allocation that is envy-free is open.

   $\bullet$  \textbf{Computation of proportional multi-allocation :} 
   Igarashi and Meunier \cite{25igarashi2021envy} show the existence of a feasible and contiguous proportional multi-allocation for any number $m$ of layers and any number $n\geq m$ of agents. Hosseini et al. \cite{20Hosseini2020FairDO} give a computational procedure of finding a feasible and contiguous multi-allocation for $n\geq m$ agents and $m=2^a$ layers, where $a$ is a positive integer. We extend the result for $n\geq m$ agents and $m$ layers, where $m=3,2^a3$ and $a$ is a positive integer. Extending the result for an arbitrary $m$ is unsolved.
   
   $\bullet$  \textbf{Existence of exact multi-allocation :} Austin
   \cite{29austin1982sharing} shows the existence of an exact allocation for a single-layered cake and two agents. We show the existence of an exact feasible multi-allocation for two layers and two agents. Alon \cite{alon1987splitting} shows the existence of an exact allocation for a single-layered cake. Finding the existence of an exact feasible multi-allocation for $n\geq m$ agents and $m$ layers is open. 

   $\bullet$ \textbf{Efficiency of multi-allocation :} Caragiannis et al. \cite{7caragiannis2012efficiency} explore how allocation efficiency is affected by fairness. They take into account three distinct concepts of fairness for the allocations of divisible and indivisible goods and chores: proportionality, envy-freeness, and equitability. In comparison to optimal allocations, fair allocations lose efficiency. They quantify this loss and demonstrate the price of justice under three different concepts. Aumann and Dombb \cite{6aumann2015efficiency} examine how fairness criteria may lead to a decrease in social welfare, focusing mostly on a scenario where each agent requires a connected piece. The study of efficiency in multi-layered cakes is another direction of work.
\section*{Acknowledgements}
The author would like to thank Bodhayan Roy for his suggestions, which have improved the presentation of this paper significantly. 
\section*{Declaration of competing interest}
The author declares that he has no known competing financial interests or personal relationships that could have appeared to influence the work reported in this paper.
\section*{Data availability}
No data was used for the research described in the article.

\section*{Appendix}
\subsection*{Exact multi-layered cake cutting}
\subsubsection*{The m (even)-layered cut} Initially, the issue of why we define another division is highlighted. In Section $\ref{sec:exact}$, our goal is to determine an exact multi-allocation that is feasible. When both agents have the same valuation functions, corresponding to specific to this case, we are able to obtain an exact multi-allocation that is both feasible and contiguous. Otherwise, we never get an exact multi-allocation that meets the two criteria. The following question therefore arises: Is there an exact multi-allocation that satisfies the condition of feasibility?
Suppose that there is an exact multi-allocation that meets the feasibility condition and that both agents have different valuation functions. In light of this, the multi-allocation $\mathcal{P}=(\mathcal{P}_1,\mathcal{P}_2)$  that meets the feasibility constraint may have the form $\mathcal{P}_1=(([0,x]\cap L_1)\cup([y,1] \cap L_1) ,[x,y]\cap L_2)$ and $\mathcal{P}_2=([x,y]\cap L_1,([0,x]\cap L_2)\cup ([y,1]\cap L_2))$.
\begin{center}
    \begin{tikzpicture}
         \draw[thick,-] (0,0)--(5,0);
         \draw[-](0,1)--(5,1);
         \draw[-](0,0)--(0,1);
         \draw[-](5,0)--(5,1);
         \draw[-](0,.5)--(5,.5);
         \draw[-](1.8,0)--(1.8,1);
         \draw[-](3.5,0)--(3.5,1);
         \node at (.9,.75) {$P_{11}$};
         \node at (.9,.25) {$P_{22}$};
         \node at (2.65,.25) {$P_{12}$};
         \node at (2.65,.75){$P_{21}$};
         \node at (4.25,.75){$P_{22}$};
         \node at (4.25,.25){$P_{11}$};
         \node at (-.4,.25){$L_2$};
         \node at (-.4,.75){$L_1$};
         \node at (1.8,1.2){$x$};
         \node at (3.5,1.2){$y$};
         \node at (2.5,-.3) {Figure 17 : $\mathcal{P}=(\mathcal{P}_1,\mathcal{P}_2)$};
\end{tikzpicture} 
\end{center}
Our defined partition and computational model are very relevant to the switching point property and the Austin moving-knife procedure.

In the Austin moving-knife procedure, we start by moving two knives over a single-layered cake from positions 0 and p so that the piece between the two knives is always half with respect to agent 1, where p divides the cake into equal-valued pieces with respect to agent 1.
\begin{center}
    \begin{tikzpicture}
        \draw[blue, very thick] (0,0) rectangle (5,1);
        \draw[black, very thick] (0,-1)--(0,2);
        \draw[black, very thick] (2,-1)--(2,2);
         \node at (-0.2,1.2) {$0$};
         \node at (2.2,1.2) {$p$};
         \node at (5,1.2) {$1$};
         \node at (0,-1.2) {1st Knife};
         \node at (2,-1.2) {2nd knife};
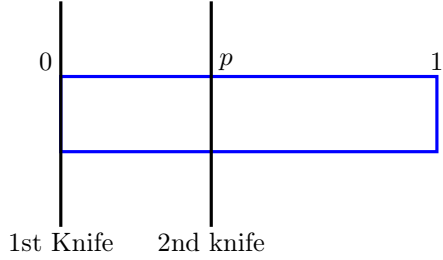
         \node at (2.5,-1.7) {Figure 18 : Initial positions of the two knives.};
    \end{tikzpicture}
\end{center}
The movement of those two knives will come to an end at points $p$ and $1$, respectively, since point $p$ divides the cake into two equal-valued portions with regard to agent 1.
\begin{center}
    \begin{tikzpicture}
        \draw[blue, very thick] (0,0) rectangle (5,1);
        \draw[black, very thick] (5,-1)--(5,2);
        \draw[black, very thick] (2,-1)--(2,2);
         \node at (0,1.2) {$0$};
         \node at (2.2,1.2) {$p$};
         \node at (5.2,1.2) {$1$};
         \node at (2,-1.2) {1st knife};
         \node at (5,-1.2) {2nd knife};
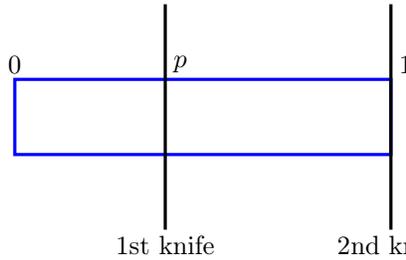
         \node at (2.5,-1.7) {Figure 19 : Terminal positions of the two knives.};
    \end{tikzpicture}
\end{center}
When the second agent believes the value of the piece between the two knives is half, he will order the movement of the knives to stop. The intermediate value theorem demands that this situation occur.
\begin{center}
    \begin{tikzpicture}
        \draw[blue, very thick] (0,0) rectangle (5,1);
        \draw[black, very thick] (4,-1)--(4,2);
        \draw[black, very thick] (1.5,-1)--(1.5,2);
         \node at (0,1.2) {$0$};
         \node at (1.7,1.2) {$x$};
         \node at (4.2,1.2) {$y$};
         \node at (5,1.2) {$1$};
         \node at (1.5,-1.2) {1st knife};
         \node at (4,-1.2) {2nd knife};
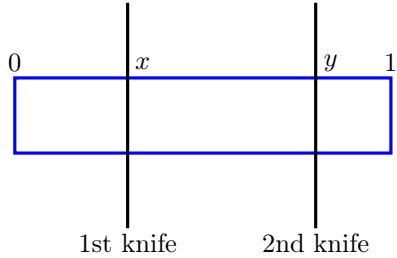
         \node at (2.5,-1.7) {Figure 20 : Positions of the two knives when exactness happens.};
    \end{tikzpicture}
\end{center}
\subsubsection*{Our Procedure to show the existence of an exact feasible multi-allocation}
Similar to the Austin moving-knife procedure, we define a pair of knives to divide the cake under feasibility constraints. Suppose that $s_1$ is a switching point of agent 1. Due to the definition of the switching point, we obtain $V_1(LR(s_1))=V_1(RL(s_1))=\nicefrac{1}{2}$. Therefore, $V_1(TLR(0,s_1))=V_1(RL(s_1))=\nicefrac{1}{2}$ and $V_1(TLR(s_1,1))=V_1(LR(s_1))=\nicefrac{1}{2}$.
\begin{center}
    \begin{tikzpicture}
         \draw[thick,-] (0,0)--(3,0);
    \draw[-](0,1)--(3,1);
    \draw[-](0,.5)--(3,.5);
    \draw[-](0,0)--(0,1);
    \draw[-](3,0)--(3,1);
    \draw[-](1.5,0)--(1.5,1);
    \node at (0,1.2) {$x=0$};
    \node at (1.5,1.2){$y=s_1$};
    \node at (-.4,.25){$L_2$};
    \node at (-.4,.75){$L_1$};
    \node at (.75,.75){$LR(s_1)$};
    \node at (.75,.25){$RL(s_1)$};
    \node at (2.25,.25){$LR(s_1)$};
    \node at (2.25,.75){$RL(s_1)$};
    \node at (1.5,-.3) {Figure 21: Partition for the pair $(0,s_1)$.};
    \end{tikzpicture}
\end{center}
\subsubsection*{Explanation}
Similar to the Austin moving-knife procedure, we begin to move the pair of knives $K_1$ and $K_2$ from positions $0$ and $s_1$ such that the value of the piece $TLR(0, s_1)$ is half for agent 1, where $s_1$ is a switching point of agent 1.
\begin{center}
    \begin{tikzpicture}
         \draw[thick,-] (0,0)--(5,0);
         \draw[-](0,1)--(5,1);
         \draw[blue,-](0,-1)--(0,2);
         \draw[-](5,0)--(5,1);
         \draw[-](0,.5)--(5,.5);
         \draw[blue,-](2.5,-1)--(2.5,2);
         \node at (-.4,.25){$L_2$};
         \node at (-.4,.75){$L_1$};
         \node at (.4,1.2) {$x=0$};
         \node at (3,1.2) {$y=s_1$};
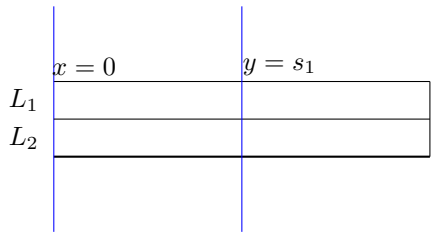
         \node at (2.5,-1.7) {Figure 22: Initial positions of the pair of knives};
\end{tikzpicture} 
\end{center}

We continuously move these two knives such that the piece $TLR(x,y)$ is valued at half with respect to agent 1, where $x\leq y$. It is always possible to move the pair of knives while keeping the value on the piece $TLR(x,y)$ at half with respect to agent 1 because we simultaneously move the pair of knives according to the valuations of agent 1. The terminal locations of the knives $K_1$ and $K_2$ are at points $s_1$ and $1$, respectively. 
\begin{center}
    \begin{tikzpicture}
         \draw[thick,-] (0,0)--(5,0);
         \draw[-](0,1)--(5,1);
         \draw[blue,-](5,-1)--(5,2);
         \draw[-](0,0)--(0,1);
         \draw[-](0,.5)--(5,.5);
         \draw[blue,-](2.5,-1)--(2.5,2);
         \node at (-.4,.25){$L_2$};
         \node at (-.4,.75){$L_1$};
         \node at (5.4,1.2) {$y=1$};
         \node at (3,1.2) {$x=s_1$};
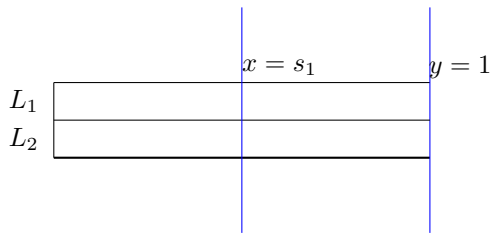
         \node at (2.5,-1.7) {Figure 23: Terminal positions of the pair of knives};
\end{tikzpicture} 
\end{center}

While we continuously move these two knives, keeping the value on the piece $TLR(x,y)$ always half with regard to agent 1, agent 2 can sometimes say "stop" when he thinks the value on the piece $TLR(x,y)$ is half. The intermediate value theorem implies that this situation occurs.
\begin{center}
    \begin{tikzpicture}
         \draw[thick,-] (0,0)--(5,0);
         \draw[-](0,1)--(5,1);
         \draw[-](0,0)--(0,1);
         \draw[-](5,0)--(5,1);
         \draw[-](0,.5)--(5,.5);
         \draw[blue,-](1.8,-1)--(1.8,2);
         \draw[blue,-](3.5,-1)--(3.5,2);
         \node at (.9,.75) {$TLR$};
         \node at (.9,.25) {$TRL$};
         \node at (2.65,.25) {$TLR$};
         \node at (2.65,.75){$TRL$};
         \node at (4.25,.75){$TLR$};
         \node at (4.25,.25){$TRL$};
         \node at (-.4,.25){$L_2$};
         \node at (-.4,.75){$L_1$};
         \node at (1.6,1.2){$x$};
         \node at (3.7,1.2){$y$};
         \node at (1.8,-1.3) {$K_1$};
         \node at (3.8,-1.3) {$K_2$};
         \node at (2.5,-1.7) {Figure 24: Exact multi-allocation};
\end{tikzpicture} 
\end{center} 
 
\bibliographystyle{unsrt}
\bibliography{references}  

\begin{thebibliography}{10}

\bibitem{1brams1996fair}
Steven~J Brams and Alan~D Taylor.
\newblock {\em Fair Division: From cake-cutting to dispute resolution}.
\newblock Cambridge University Press, 1996.

\bibitem{4moulin2004fair}
Herv{\'e} Moulin.
\newblock {\em Fair division and collective welfare}.
\newblock MIT press, 2004.

\bibitem{2robertson1998cake}
Jack Robertson and William Webb.
\newblock {\em Cake-cutting algorithms: Be fair if you can}.
\newblock CRC Press, 1998.

\bibitem{3brandt2016handbook}
Felix Brandt, Vincent Conitzer, Ulle Endriss, J{\'e}r{\^o}me Lang, and Ariel~D
  Procaccia.
\newblock {\em Handbook of computational social choice}.
\newblock Cambridge University Press, 2016.

\bibitem{5edmonds2011cake}
Jeff Edmonds and Kirk Pruhs.
\newblock Cake cutting really is not a piece of cake.
\newblock {\em ACM Transactions on Algorithms (TALG)}, 7(4):1--12, 2011.

\bibitem{6aumann2015efficiency}
Yonatan Aumann and Yair Dombb.
\newblock The efficiency of fair division with connected pieces.
\newblock {\em ACM Transactions on Economics and Computation (TEAC)},
  3(4):1--16, 2015.

\bibitem{7caragiannis2012efficiency}
Ioannis Caragiannis, Christos Kaklamanis, Panagiotis Kanellopoulos, and Maria
  Kyropoulou.
\newblock The efficiency of fair division.
\newblock {\em Theory of Computing Systems}, 50:589--610, 2012.

\bibitem{8thomson2007children}
William Thomson.
\newblock Children crying at birthday parties. why?
\newblock {\em Economic Theory}, 31(3):501--521, 2007.

\bibitem{9procaccia2013cake}
Ariel~D Procaccia.
\newblock Cake cutting: Not just child's play.
\newblock {\em Communications of the ACM}, 56(7):78--87, 2013.

\bibitem{10branzei2022query}
Simina Br{\^a}nzei and Noam Nisan.
\newblock The query complexity of cake cutting.
\newblock {\em Advances in Neural Information Processing Systems},
  35:37905--37919, 2022.

\bibitem{11brams1995envy}
Steven~J Brams and Alan~D Taylor.
\newblock An envy-free cake division protocol.
\newblock {\em The American Mathematical Monthly}, 102(1):9--18, 1995.

\bibitem{12edward1999rental}
Francis Edward~Su.
\newblock Rental harmony: Sperner's lemma in fair division.
\newblock {\em The American mathematical monthly}, 106(10):930--942, 1999.

\bibitem{13stromquist1980cut}
Walter Stromquist.
\newblock How to cut a cake fairly.
\newblock {\em The American Mathematical Monthly}, 87(8):640--644, 1980.

\bibitem{14dubins1961cut}
Lester~E Dubins and Edwin~H Spanier.
\newblock How to cut a cake fairly.
\newblock {\em The American Mathematical Monthly}, 68(1P1):1--17, 1961.

\bibitem{15aziz2016discrete}
Haris Aziz and Simon Mackenzie.
\newblock A discrete and bounded envy-free cake cutting protocol for four
  agents.
\newblock In {\em Proceedings of the forty-eighth annual ACM symposium on
  Theory of Computing}, pages 454--464, 2016.

\bibitem{16aziz2016discrete}
Haris Aziz and Simon Mackenzie.
\newblock A discrete and bounded envy-free cake cutting protocol for any number
  of agents.
\newblock In {\em 2016 IEEE 57th Annual Symposium on Foundations of Computer
  Science (FOCS)}, pages 416--427. IEEE, 2016.

\bibitem{17}
Haris Aziz and Simon Mackenzie.
\newblock A bounded and envy-free cake cutting algorithm.
\newblock {\em Commun. ACM}, 63(4):119–126, 2020.

\bibitem{18even1984note}
Shimon Even and Azaria Paz.
\newblock A note on cake cutting.
\newblock {\em Discrete Applied Mathematics}, 7(3):285--296, 1984.

\bibitem{19stromquist2008envy}
Walter Stromquist.
\newblock Envy-free cake divisions cannot be found by finite protocols.
\newblock {\em the electronic journal of combinatorics}, 15(1):R11, 2008.

\bibitem{20Hosseini2020FairDO}
Hadi Hosseini, Ayumi Igarashi, and Andrew Searns.
\newblock Fair division of time: Multi-layered cake cutting.
\newblock In {\em International Joint Conference on Artificial Intelligence},
  2020.

\bibitem{21caragiannis2019unreasonable}
Ioannis Caragiannis, David Kurokawa, Herv{\'e} Moulin, Ariel~D Procaccia,
  Nisarg Shah, and Junxing Wang.
\newblock The unreasonable fairness of maximum nash welfare.
\newblock {\em ACM Transactions on Economics and Computation (TEAC)},
  7(3):1--32, 2019.

\bibitem{22procaccia2009thou}
Ariel~D Procaccia.
\newblock Thou shalt covet thy neighbor's cake.
\newblock In {\em Twenty-First International Joint Conference on Artificial
  Intelligence}, 2009.

\bibitem{23deng2012algorithmic}
Xiaotie Deng, Qi~Qi, and Amin Saberi.
\newblock Algorithmic solutions for envy-free cake cutting.
\newblock {\em Operations Research}, 60(6):1461--1476, 2012.

\bibitem{24barbanel2005geometry}
Julius~B Barbanel.
\newblock {\em The geometry of efficient fair division}.
\newblock Cambridge University Press, 2005.

\bibitem{Aumann2012ComputingSC}
Yonatan Aumann, Yair Dombb, and Avinatan Hassidim.
\newblock Computing socially-efficient cake divisions.
\newblock {\em AAMAS}, 2012.

\bibitem{25igarashi2021envy}
Ayumi Igarashi and Fr{\'e}d{\'e}ric Meunier.
\newblock Envy-free division of multi-layered cakes.
\newblock In {\em International Conference on Web and Internet Economics},
  pages 504--521. Springer, 2021.

\bibitem{26lebert2013envy}
Nicolas Lebert, Fr{\'e}d{\'e}ric Meunier, and Quentin Carbonneaux.
\newblock Envy-free two-player m-cake and three-player two-cake divisions.
\newblock {\em Operations Research Letters}, 41(6):607--610, 2013.

\bibitem{27nyman2020fair}
Kathryn Nyman, Francis~Edward Su, and Shira Zerbib.
\newblock Fair division with multiple pieces.
\newblock {\em Discrete Applied Mathematics}, 283:115--122, 2020.

\bibitem{28cloutier2010two}
John Cloutier, Kathryn~L Nyman, and Francis~Edward Su.
\newblock Two-player envy-free multi-cake division.
\newblock {\em Mathematical Social Sciences}, 59(1):26--37, 2010.

\bibitem{29austin1982sharing}
A~Keith Austin.
\newblock Sharing a cake.
\newblock {\em The Mathematical Gazette}, 66(437):212--215, 1982.

\bibitem{alon1987splitting}
Noga Alon.
\newblock Splitting necklaces.
\newblock {\em Advances in Mathematics}, 63(3):247--253, 1987.

\end{thebibliography}

\end{document}